\documentclass[conference]{IEEEtran}
\IEEEoverridecommandlockouts

\usepackage{cite}
\usepackage{times}
\usepackage{epsfig}
\usepackage{graphicx}
\usepackage{amsmath}
\usepackage{amssymb}
\usepackage{booktabs}
\usepackage{algorithm}
\usepackage{algorithmic}
\usepackage{booktabs}
\usepackage{svg}
\usepackage{subcaption}
\usepackage{makecell}
\usepackage{tikz}
\usepackage{times}
\usepackage{epsfig}
\usepackage{mathtools}
\usepackage{amsthm}
\usepackage{url} 

\theoremstyle{plain}
\newtheorem{theorem}{Theorem}[section]
\newtheorem{proposition}[theorem]{Proposition}
\newtheorem{lemma}[theorem]{Lemma}
\newtheorem{corollary}[theorem]{Corollary}
\theoremstyle{definition}

\theoremstyle{remark}
\newtheorem{remark}[theorem]{Remark}

\begin{document}

\title{Adaptive Consensus Optimization Method for GANs\footnote{Accepted in IJCNN 2023}}

\author{\IEEEauthorblockN{1\textsuperscript{st} Given Name Surname}
\IEEEauthorblockA{\textit{dept. name of organization (of Aff.)} \\
\textit{name of organization (of Aff.)}\\
City, Country \\
email address or ORCID}
}

\author{\IEEEauthorblockN{Anonymous Authors}}
\author{\IEEEauthorblockN{Sachin Kumar Danisetty}
\IEEEauthorblockA{
\textit{IIIT Hyderabad, India}\\
sachin.danisetty@alumni.iiit.ac.in}
\and
\IEEEauthorblockN{Santhosh Reddy Mylaram}
\IEEEauthorblockA{
\textit{IIIT Hyderabad, India}\\
santhosh.reddy@alumni.iiit.ac.in}
\and 
\IEEEauthorblockN{Pawan Kumar}
\IEEEauthorblockA{
\textit{IIIT Hyderabad, India}\\
pawan.kumar@iiit.ac.in}
}

\maketitle

\begin{abstract}
  We propose a second order gradient based method with ADAM and RMSprop for the training of generative adversarial networks. The proposed method is fastest to obtain similar accuracy when compared to prominent second order methods. Unlike state-of-the-art recent methods, it does not require solving a linear system, or it does not require additional mixed second derivative terms. We derive the fixed point iteration corresponding to proposed method, and show that the proposed method is convergent. The proposed method produces better or comparable inception scores, and comparable quality of images compared to other recently proposed state-of-the-art second order methods. Compared to first order methods such as ADAM, it produces significantly better inception scores. The proposed method is compared and validated on popular datasets such as FFHQ, LSUN, CIFAR10, MNIST, and Fashion MNIST for image generation tasks\footnote{Accepted in IJCNN 2023}. Codes: \url{https://github.com/misterpawan/acom}
\end{abstract}

\section{Introduction and Related Work.}
Recently generative modeling have received much attention with the advent of diffusion based models \cite{dhariwal2021} for text-to-image generation, however, sampling for generative adversarial networks (GAN) \cite{goodfellow2014} still remain order of magnitude fast. Moreover, with bigger architectures as in \cite{kang2023gigagan}, image quality from GAN remain as good as those from diffusion models. We consider the problem of solving the following min max problem 
\begin{align}
    \min_x \max_y f(x,y),  \label{minmax}
\end{align}
where $f:\mathbb{R}^m \times \mathbb{R}^n \rightarrow \mathbb{R}.$ This can be seen as a two-player game, where one agent tries to maximize its objective, whereas, the other agent tries to minimize its objective. 

In this work, we are interested in such optimization problems stemming from generative adversarial networks. Such problems also arise in adversarial training \cite{madry2019deep} and multi-agent reinforcement learning \cite{omidshafiei2017deep}. This is an active area of research, and recent solvers often involve second order derivatives in some way. Using second order derivatives is found to increase robustness and quality of the images generated. The minmax problem \eqref{minmax} above can be seen either as simultaneous minmax or sequential minmax problem. If it is indeed seen as simultaneous minmax, then the solutions will correspond to local Nash equilibrium \cite{jin2019,schaefer2019,schaefer20a}. However, in the current literature, there is no well established clarity on usage of one of these in designing solver for GAN. As pointed out in \cite{jin2019}, it is understood that GANs correspond to sequential min-max, i.e., generator observes the discriminator's action, then generator optimizes, followed by discriminator rather than both generator and discriminator optimizing simultaneously. These two views of simultaneous or sequential minmax leads to a variety of second order gradient methods. 

A class of methods consider minmax interpretation; here it is understood that if the discriminator is optimal, then it must lead the generator loss to approach Jensen-Shannon divergence between real and generated distribution. This view leads to class of methods that suggest us to use variety of divergences or metrics with improved theoretical properties. As was pointed out in \cite{schaefer20a}, the minmax interpretation has two major problem: ``Without regularity constraints, the discriminator can always be perfect" and ``Imposing regularity constraints needs a measure of similarity of images." They claim that imposing regularity is equivalent to forcing the discriminator to map images to similar images. However, it is not easy to mimic similarity of images using such map that corresponds to similarity seen by human perception. Furthermore, in \cite{lucic2017}, authors did not find significant differences in the performance of GANs with various choices of divergence measures. In \cite{mazumdar2018}, it was shown that simultaneous gradient descent (SimGD) on both players leads to additional stable points compared to the case when gradient descent is done sequentially, i.e., by considering one of the player fixed at a time. Moreover, these additional stable points do not correspond to local Nash equilibrium. Considering this concern of unusual additional stable points, recent approaches such as \cite{mazumdar2019} and \cite{balduzzi2018} suggest modifications that lead only to local Nash equilibrium. Another class of methods stems from the original GAN \cite{goodfellow2014}, called metric-agnostic GANs. In original GAN, the loss function is given as follows:
\begin{align*}
    \min_{\mathcal{G}} \max_{\mathcal{D}} \dfrac{1}{2} \mathbb{E}_{x \sim P_{\text{data}}} [\log \mathcal{D}(x)] + \dfrac{1}{2} \mathbb{E}_{x \sim P_G} [\log (1 - \mathcal{D}(x))],
\end{align*}
where $\mathcal{G}$ is the distribution generated by generator, and $\mathcal{D}$ is the classifier provided by discriminator, and $P_{\text{data}}$ is the target. In \cite{arjovsky2017,arjovsky2018}, WGAN was proposed with the following loss function 
\begin{align*}
    \min_{\mathcal{G}} \max_{\mathcal{D}} \mathbb{E}_{x \sim P_{\text{data}}} [\mathcal{D}(x)] - \mathbb{E}_{x \sim P_{\mathcal{G}}} [\mathcal{D}(x)] + \mathcal{F}(\nabla \mathcal{D}), 
\end{align*}
where $\mathcal{F}(\nabla \mathcal{D})$ is infinity if $\sup_x \| \nabla \mathcal{D}(x) \|>1$ and zero elsewhere. Shortly later, WGAN-GP  was proposed in \cite{gulrajani2017}, where the inequality constraint is replaced by $\mathcal{F}(\nabla \mathcal{D}) = \mathbb{E}[(\| \nabla \mathcal{D} \| - 1)^2].$ These methods depend on the choice of norm used to measure gradient $\nabla \mathcal{D}.$ There are quite a few other variants with other measures or norms proposed in Banach-GAN \cite{adler2018}, Sobolev-GAN \cite{mroueh2017}, and Besov-GAN \cite{uppal2019} to measure $\nabla \mathcal{D}.$ These are so-called metric-informed GANs.   

For solving problem \eqref{minmax}, several solvers have been proposed in the past. A straightforward approach is gradient descent-ascent (GDA), in this case, the two players see each of their objectives separately as minimization problem without any regard to the other player's interest. It is well documented that this approach may lead to a cycling behavior \cite{schaefer2019}. Hence, GDA is not a suitable solver for competitive optimization as seen in GANs. In \cite{goodfellow2014}, two scale update rule is proposed, methods that use follow the regularized leader is proposed in \cite{BERGER2007572}, predictive approach is shown in \cite{yadav2017}, solver based on opponent learning awareness is proposed in \cite{foerster2019}. Similarly, some of the sophisticated heuristics based on one agent predicting other agent's next move was proposed in \cite{daskalakis2017,merti2018,rakhlin2013}.     
 In \cite{schaefer2019,schaefer20a}, authors proposed a new method, namely, CGD, for numerical solution of \eqref{minmax}. Compared to some of the methods mentioned before, the CGD method avoids divergence or oscillations, which is typical of some of the methods based on alternating gradient descent. In their later paper \cite{schaefer20a}, 
 they claim that the ACGD \cite{schaefer20a} method provides implicit competitive regularization \cite{arora2019,azizan2019,gunasekar2017,jin2019,ma2017,neyshabur2017}. In \cite{Mescheder2018ICML}, it is shown that unregularized GAN is neither locally nor globally convergent. Some of the above methods could be seen in the framework of preconditioned gradient methods used in other applications \cite{orseau2019preconditioned,Das_2021_WACV,9093265,das2020,kumar2013,BENZI2002418,KUMAR20132251,kumar2014,saad2003iterative,kumar2013b,kumar2015a,Kumar2016c,6900213,kumar2011purely,kumar2010b}. 

In Table \ref{table:updates}, we see update procedures for various algorithms, we notice that, unlike others, an approximate version of CGD involves solving for a linear system. We also notice that except for GDA, all other methods make use of second order terms.  

\tiny
\begin{table}
\begin{center}
\begin{tabular}[H]{l|l}
\toprule 
 Update rule &   Name \\ 
 \hline 
    $\Delta x$ =  $- \nabla_x f$ &  GDA\\
    $\Delta x$ =  $- \nabla_x f - \gamma D_{xy}^2 f \nabla_y f$ & SGA \cite{balduzzi2018} \\
    $\Delta x$ =  $- \nabla_x f - \gamma D_{xy}^2 f \nabla_y f - \gamma D_{xx}^2 f \nabla_x f$ & ConOpt \cite{mescheder2017}\\
    $\Delta x$ =  $- \nabla_x f - \gamma D_{xy}^2 f \nabla_y f + \gamma D_{xx}^2 f \nabla_x f$ & OGDA \cite{daskalakis2017} \\
     $\Delta x$ =  $- (Id + \eta^2 D_{xy}^2 f D_{yx}^2 f)^{-1}$   
    $\left( \nabla_x f - \gamma D_{xy}^2 f \nabla_y f \right)$ & CGD \cite{schaefer2019} \\
    $\Delta x$ = $-\nabla_x f - D_{xx}^2 f \Delta x$ & ACOM \\ 
    \hline 
\end{tabular}
\end{center}
\caption{\label{table:updates}Various update rules for min-max optimization problems.}
\end{table}
\normalsize

\paragraph{Contributions.} In this paper, we propose a new simple update rule for solving \eqref{minmax}. The image quality generated by the proposed method is among the best second order methods, moreover, it is fastest to train compared to all second order methods we compared with. The proposed update rule for the gradients is integrated with RMSprop and ADAM for adaptive learning rate. In particular, we observe that mixed derivative terms as used in existing methods does not seem to be necessary on practical datasets such as MNIST, Fashion MNIST, CIFAR10, FFHQ, and LSUN. The mixed derivative terms in some of the existing methods are motivated from the view that solution corresponds to local Nash equilibrium, however, we don't find it much useful in practice. Moreover, we also show theoretical guarantee for convergence of the proposed method. 

We summarize the main contributions of the paper as follows:
\begin{itemize}
    \item We propose a new second order method called adaptive consensus optimization (ACOM) integrated with adaptive learning rate of RMSProp or ADAM. We show that the proposed method is fastest to train among second order methods, and achieves inception scores as good as existing state-of-the-art. Extensive experiments on five popular datasets are shown.
    \item We show a complete convergence analysis of our method with RMSprop and ADAM. We identify the fixed point iteration, and show that the necessary condition for convergence of the proposed method is satisfied, ensuring at least linear convergence. Although, analysis was done for ConOpt alone in \cite{mescheder2017} and for CGD alone in \cite{schaefer2019}, unified full analysis of the update rule combined with momentum based methods such as ADAM or RMSprop is shown for the first time.  
\end{itemize}

The following sections are organized as follows. In section 2, we describe briefly the GAN and smooth two-player game. In Section 3, we describe the proposed method ACOM, and convergence of the proposed method. Finally, in Section 4, we show the numerical experiments on five popular  datasets MNIST, Fashion MNIST, CIFAR10, LSUN, and FFHQ.

\section{GAN and smooth two-player games.}
We wish to find a Nash-equilibrium of the two player game associated with training GAN. We define Nash-equilibrium as a point $\bar{p} = (\bar{x}, \bar{y})$ if the following two conditions hold
\begin{align*}
    \bar{x} \in \text{arg max}_x f(x, \bar{y}) \quad \text{and} \quad \bar{y} \in \text{arg max}_y g(\bar{x}, y)  
\end{align*}
in some local neighborhood of $(\bar{x}, \bar{y}).$ For differentiable two-player game, associated vector field is given by 
$    V(x, y) = \begin{bmatrix}
        D_x f(x, y) \\
        D_y g(x, y)
    \end{bmatrix},
$    
where 
$$D_xf = \nabla_x f, \: D_yf = \nabla_y g.$$
For a zero sum game, we have $f = -g,$ and the derivative of the vector field is 
\begin{align*}
    V'(x, y) = \begin{bmatrix}
        D_{xx}^{2} f(x, y) & D_{xy}^{2} f(x,y) \\
        -D_{yx}^{2} f(x,y) & -D_{yy}^{2} f(x,y) 
    \end{bmatrix},
\end{align*}
where
 $$   D_{xx}^{2}f = \nabla^2_{xx}f, \: D_{xy}^{2}f = \nabla_{xy}^2f, \: D_{yy}^{2}f = \nabla_{yy}^2f.$$
\begin{lemma}
\label{lem:ifandonlyif}
For zero-sum games, $V'(p)$ is negative semi-definite if and only if $D_{xx}^{2}f(x,y)$ is negative semi-definite and $D_{yy}^{2}f(x,y)$ is positive semi-definite.    
\end{lemma}
\begin{proof}
See \cite{mescheder2017}.
\end{proof}
\begin{corollary} \label{cor:negdef}
For zero-sum games, $V'(p)$ is negative semi-definite for any local Nash-equilibrium $\bar{p}.$ Conversly, if $\bar{p}$ is a stationary point of $V(p)$ and $V'(\bar{p})$ is negative-definite, then $\bar{p}$ is a local Nash-equilibrium.  
\end{corollary}
\begin{proof}
See \cite{mescheder2017}.
\end{proof}

\subsection{Results for Fixed Point Iteration.}
To analyze the convergence properties of our proposed method, we begin with the classical theorem for convergence of fixed point iterations: 

\begin{proposition}
\label{prop:fixedpoint}
Let $F: \Omega \rightarrow \Omega$ be a continuously differential function on an open subset $\Omega$ of $\mathbb{R}^{n}$ and let $\bar{p} \in \Omega$ be so that
\begin{enumerate}
    \item $F(\bar{p})=\bar{p}$, and
    \item The absolute values of the eigenvalues of the Jacobian $F^{\prime}(\bar{p})$ are all smaller than 1 .
\end{enumerate}
Then there is an open neighborhood $U$ of $\bar{p}$ so that for all $p_{0} \in U$, the iterates $F^{(k)}\left(p_{0}\right)$ converge to $\bar{p}$. The rate of convergence is at least linear. More precisely, the error $\left|F^{(k)}\left(p_{0}\right)-\bar{p}\right|$ is in $\mathcal{O}\left(\left|\lambda_{\max }\right|^{k}\right)$ for $k \rightarrow \infty,$ where $\lambda_{\max }$ is the eigenvalue of $F^{\prime}(\bar{p})$ with the largest absolute value.
\end{proposition}
\begin{proof}
See \cite{Bertsekas99}, proposition 4.4.1.
\end{proof}

\begin{lemma}
\label{eqn:boundh}
Assume that $A \in \mathbb{R}^{n \times n}$ only has eigenvalues with negative real-part and let $h>0$. Then the eigenvalues of the matrix $I+hA$  lie in the unit ball if and only if
\begin{align*}
h<\frac{1}{|\Re(\lambda)|} \frac{2}{1+\left(\frac{\Im(\lambda)}{\Re(\lambda)}\right)^{2}}.
\end{align*}
\end{lemma}
\begin{proof}
See Lemma 4 of \cite{mescheder2017}. 
\end{proof}

For the choice of $F(p) = p + hG(p)$ for some $h>0,$ the Jacobian is given by
$F'(p) = I + hG'(p).$ Hence, using lemma \ref{eqn:boundh} above for $A = G'(p),$ 
we claim convergence using proposition \ref{prop:fixedpoint}. 

\section{ACOM: Adaptive Consensus Optimization Method.}
The proposed method is derived as follows. Consider the partial derivative defined as follows
\begin{align}
    \nabla_x f(x,y) \Delta x = f(x+\Delta x, y) - f(x,y). \label{deriv}
\end{align}
Taking partial derivative with respect to $x,$ we have 
\begin{align*}
    \nabla_{xx}^2 f(x,y) \Delta x = \nabla_x f(x + \Delta x, y) - \nabla_x f(x,y),
\end{align*}
which leads to the following update to the gradient at the new point $x + \Delta x:$
\begin{align}
    \nabla_x f(x + \Delta x, y) = \nabla_x f(x,y) + \nabla_{xx}^2 f(x,y) \Delta x. \label{updatex}
\end{align}
Similarly, new update to the gradient with respect to variable $y$ would be 
\begin{align}
    \nabla_y f(x, y + \Delta y) = \nabla_y f(x,y) + \nabla_{yy}^2 f(x,y) \Delta y. \label{updatey}
\end{align}
That is, the updates \eqref{updatex} and \eqref{updatey} can be seen as first order Taylor series expansion of $\nabla_x f(x+ \Delta x, y)$ and $\nabla_y f(x, y + \Delta y)$ at $x+\Delta x$ and $y + \Delta y$ respectively. As we will see in numerical experiments, these are simple update rules and are as effective as ACGD \cite{schaefer20a} in obtaining high inception scores, while it is much faster than ACGD.  
The full algorithm is shown in Algorithm \ref{acom}. In line number 5 and 12, the update rules for gradients described above are used. As we notice, we supply the updated gradients to the ADAM method, which subsequently uses these to compute first momentum and second momentum terms. As mentioned before, compared to Algorithm CGD or ACGD, ACOM does not require expensive linear system solve, moreover, it does not use mixed derivative terms $D^2_{xy} f$ as in ConOpt \cite{mescheder2017}.  

\subsection{Convergence of ACOM.}
Our convergence proofs follow the framework of \cite{mescheder2017}, and we refer the reader to see theoretical comparisons for other methods. 
\subsubsection{Convergence of ACOM with RMSPROP} 
In ConOpt \cite{mescheder2017}, the fixed point update rules were written for second order update procedure of ConOpt, and similar analysis for momentum based method was done separately. In the following, we do a combined analysis of our second order update rule ACOM with RMSprop. To the best of our knowledge, all methods in the past were analyzed separately, however, combined analysis with momentum was not shown. The iterative update function for ACOM with RMSProp is given by $F(p),$ with $p=(x,y,v_x,v_y)$ is given as follows
\begin{align*}
F(x,y,v_x,v_y) = 
\begin{bmatrix}
    x + \frac{h(D_x f+D_{xx}^{2}f \Delta x)}{\sqrt{v_x+\epsilon}}\\
    y + \frac{h(D_y g+D_{yy}^{2}g \Delta y)}{\sqrt{v_y+\epsilon}}\\
    (1-\beta_1)v_x + \beta_1 (D_x f + D_{xx}^{2}f\Delta x)^2 \\
    (1-\beta_2)v_y + \beta_2 (D_y g+ D_{yy}^{2}g\Delta y)^2
\end{bmatrix},
\end{align*}
where $v_x$,$v_y$ are the second order momentum of gradients of $x,y$ respectively, $ 0 \leq \beta_1,\beta_2 \leq 1 $ and $ h > 0.$
The Jacobian of this update function is
$  F'(x,y,v_x,v_y) = 
 \begin{bmatrix}
     P & Q\\
     R & S
 \end{bmatrix}, \: \text{where,}
$ 

\begin{align*}
P &= 
 \begin{bmatrix}
     1 + \frac{h(D_{xx}^{2}f + D_{xxx}f \Delta x)}{\sqrt{v_x+\epsilon}} & \frac{h(D_{xy}^{2}f + D_{xxy}f \Delta x) }{\sqrt{v_x+\epsilon}} \\
  \frac{h(D_{xy}^{2}g+D_{xyy}g \Delta y)}{\sqrt{v_y+\epsilon}} & 1 + \frac{h(D_{yy}^{2}g+D_{yyy}g \Delta y)}{\sqrt{v_y+\epsilon}}  \\
 \end{bmatrix},  \\
%
Q &= 
 \begin{bmatrix}
          -\frac{3h(D_xf+D_{xx}^{2}f \Delta x)}{2(v_x + \epsilon)^{\frac{3}{2}}} & 0 \\
    0 & -\frac{3h(D_yg+D_{yy}^{2}g \Delta y)}{2(v_y + \epsilon)^{\frac{3}{2}}} \\
 \end{bmatrix}, \\ 
%
R &= 
 \begin{bmatrix}
      \begin{matrix}
  2\beta_1(D_xf+D_{xx}^{2}f\Delta x) \\ 
  \hfill{} (D_{xx}^{2}f+D_{xxx}f\Delta x) 
  \end{matrix}
  & 
  \begin{matrix}
  2\beta_1(D_xf+D_{xx}^{2}f\Delta x)\\ 
    \hfill{} (D_{xy}^{2}f + D_{xxy}f \Delta x) 
  \end{matrix}\\ \\
  \begin{matrix}    
  2\beta_2(D_yg+D_{yy}^{2}g \Delta y) \\
  \hfill{} (D_{xy}^{2}g + D_{xyy}g\Delta y) 
  \end{matrix}
  &
  \begin{matrix} 
  2\beta_2(D_yg+D_{yy}^{2}g \Delta y) \\ 
  \hfill{} (D_{yy}^{2}g + D_{yyy}g\Delta y) 
  \end{matrix}\\
 \end{bmatrix}, \\
S &= 
 \begin{bmatrix}
      1-\beta_1 & 0 \\
    0 & 1-\beta_2 \\
 \end{bmatrix}.
\end{align*}
At any fixed point $(\bar{x},\bar{y}, \bar{v}_x,\bar{v}_y)$, we have $ \Delta x = 0, \Delta y = 0, \bar{v}_x = 0$, $\bar{v}_y = 0$, $D_x$ and $D_y = 0$. 

We have
$F'(\bar{x},\bar{y},\bar{v}_x,\bar{v}_y) = I + hA,$
where 
\begin{align*}
A = 
\begin{bmatrix}
   \frac{D_{xx}^{2}f}{\sqrt{\epsilon}} & \frac{D_{xy}^{2}f}{\sqrt{\epsilon}} & 0 & 0 \\
    \frac{D_{xy}^{2}g}{\sqrt{\epsilon}} & \frac{D_{yy}^{2}g}{\sqrt{\epsilon}} & 0 & 0 \\
    0 & 0 & -\frac{\beta_1}{h} & 0 \\
    0 & 0 & 0 & -\frac{\beta_2}{h} \\ 
\end{bmatrix}.
\end{align*}
For zero-sum game, $f=-g,$ the eigen values of $A$ are the eigen values of $\frac{1}{\sqrt{\epsilon}}V'$, $-\frac{\beta_1}{h}$ and $-\frac{\beta_2}{h}$. By the assumption that $V'$ is negative definite, we have all eigen values of matrix $A$ to have negative real-part. If $h$ satisfies the bound in Lemma \ref{eqn:boundh}, then the eigen values of matrix $F'(p) = I + hA$ lies in unit ball, hence, by proposition \ref{prop:fixedpoint}, the fixed point iteration $F$ is locally convergent towards a local Nash-equilibrium $(\bar{x},\bar{y}, \bar{v}_x,\bar{v}_y)$. For the iterative method to converge, according to Lemma \eqref{lem:ifandonlyif}, the eigenvalues of $D_{xx}^{2}f$ must be less than or equal to zero and for $D_{yy}^{2}f$ it must be greater than or equal to zero as verified empirically also in Figure \ref{fig:eig}. We summarize with the following remark.
\begin{remark}
With $V'$ negative semi-definite (holds for zero-sum game), we have just shown that $A$ is negative semi-definite, hence, for choice of $h$ from lemma \ref{eqn:boundh}, this lemma shows that eigenvalues of $F' = I + h A$ lies in unit ball. Hence item 2 of proposition \ref{prop:fixedpoint} is satisfied, thereby leading to convergence with at least linear rate.
\end{remark}

\subsubsection{Convergence of ACOM with ADAM} 
Similarly, as before, the iterative update function for ACOM with ADAM is given by $F(x,y,m_x,m_y,v_x,v_y)=$
\begin{align*}
  \begin{bmatrix}
    x + \frac{hm_x}{\sqrt{v_x+\epsilon}}\\
    y + \frac{hm_y}{\sqrt{v_y+\epsilon}}\\
    (1-\beta_1)m_x + \beta_1 (D_xf + D_{xx}^{2}f\Delta x) \\  
    (1-\beta_1)m_y + \beta_1 (D_yg + D_{yy}^{2}g\Delta y) \\
    (1-\beta_2)v_x + \beta_2 (D_xf + D_{xx}^{2}f\Delta x)^2 \\
    (1-\beta_2)v_y + \beta_2 (D_yg + D_{yy}^{2}g\Delta y)^2
\end{bmatrix},
\end{align*}

where $m_x$,$m_y$ are the first order momentum of gradients and $v_x$,$v_y$ are the second order momentum of gradients of $x,y$ respectively, $ 0 \leq \beta_1,\beta_2 \leq 1 $ and $ h > 0.$
The Jacobian of this update function is $F'(x,y,m_x, m_y, v_x,v_y) =$
\begin{align*}
 \begin{bmatrix}
     P_1 & P_2 & P_3\\
     Q_1 & Q_2 & Q_3 \\
     R_1 & R_2 & R_3 
 \end{bmatrix}, 
\end{align*}
where,
\begin{align*}
  P_1 &= 
 \begin{bmatrix}
     1 & 0 \\
     0 & 1 
 \end{bmatrix}, \quad 
   P_2 = 
 \begin{bmatrix}
     \frac{h}{\sqrt{v_x+\epsilon}} & 0 \\
     0 &  \frac{h}{\sqrt{v_y+\epsilon}}
 \end{bmatrix}, \\ 
 P_3 &= 
 \begin{bmatrix}
     -\frac{hm_x}{2(v_x+\epsilon)^{\frac{3}{2}}} & 0 \\
     0 & -\frac{hm_y}{2(v_y+\epsilon)^{\frac{3}{2}}}
 \end{bmatrix},
\end{align*}

\small 
\begin{align*}
  Q_1 &= 
 \begin{bmatrix}
    \beta_1(D_{xx}^{2}f+D_{xxx}f\Delta x) & \beta_1(D_{xy}^{2}f+D_{xxy}f\Delta x) \\
    \beta_1(D_{xy}^{2}g+D_{xyy}g\Delta y) & \beta_1(D_{yy}^{2}g+D_{yyy}g\Delta y)
 \end{bmatrix}, \\ 
%
  Q_2 &= 
 \begin{bmatrix}
     1-\beta_1 & 0 \\
     0 & 1-\beta_1
 \end{bmatrix}, \quad   
 Q_3 = 
 \begin{bmatrix}
     0 & 0 \\
     0 & 0
 \end{bmatrix},
\end{align*}
\normalsize

\small 
\begin{align*}
  R_1 &= 
 \begin{bmatrix}
      \begin{matrix}
  2\beta_2(D_xf+D_{xx}^{2}f\Delta x) \\ 
  \hfill{} (D_{xx}^{2}f+D_{xxx}f\Delta x) 
  \end{matrix}
  & 
  \begin{matrix}
  2\beta_2(D_xf+D_{xx}^{2}f\Delta x)\\ 
    \hfill{} (D_{xy}^{2}f + D_{xxy}f \Delta x) 
  \end{matrix}\\ \\
  \begin{matrix}    
  2\beta_2(D_yg+D_{yy}^{2}g \Delta y) \\
  \hfill{} (D_{xy}^{2}g + D_{xyy}g\Delta y) 
  \end{matrix}
  &
  \begin{matrix} 
  2\beta_2(D_yg+D_{yy}^{2}g \Delta y) \\ 
  \hfill{} (D_{yy}^{2}g + D_{yyy}g\Delta y) 
  \end{matrix}\\
 \end{bmatrix}, \\
%
  R_2 &= 
 \begin{bmatrix}
     0 & 0 \\
     0 & 0
 \end{bmatrix}, 
 \quad   
 R_3 = 
 \begin{bmatrix}
     1-\beta_2 & 0 \\
     0 & 1-\beta_2
 \end{bmatrix}.
\end{align*}
\normalsize

Again, at any fixed point $(\bar{x},\bar{y},\bar{m}_x,\bar{m}_y,\bar{v}_x,\bar{v}_y)$, we have $ \Delta x = 0, \Delta y = 0, \bar{m}_x = 0$, $\bar{m}_y = 0$, $\bar{v}_x = 0$, $\bar{v}_y = 0$, $D_xf$ and $D_yf = 0$. 

Writing 
$
F'(\bar{x},\bar{y},\bar{m}_x,\bar{m}_y,\bar{v}_x,\bar{v}_y) = I + hA,
$
where, 
\begin{align}
\label{eqn:matrixA}
\small
  A = 
\begin{bmatrix}
    0 & 0 & \frac{1}{\sqrt{\epsilon}} & 0 & 0 & 0 \\
    0 & 0 & 0 & \frac{1}{\sqrt{\epsilon}} & 0 & 0 \\
    \frac{\beta_1}{h}D_{xx}^{2}f & \frac{\beta_1}{h}D_{xy}^{2}f & \frac{-\beta_1}{h} & 0 & 0 & 0 \\
    \frac{\beta_1}{h}D_{xy}^{2}g & \frac{\beta_1}{h}D_{yy}^{2}g & 0 & \frac{-\beta_1}{h} & 0 & 0 \\
    0 & 0 & 0 & 0 & \frac{-\beta_2}{h} & 0 \\
    0 & 0 & 0 & 0 & 0 & \frac{-\beta_2}{h} 
\end{bmatrix}. 
\end{align}
Unlike RMSProp, here the condition that the real part of eigenvalues of $A$ is negative definite is not easy to show without further assumptions. We leave this as future work.

\subsection{Comparison of Computational Complexity.}
Assuming $x \in \mathbb{R}^m, y \in \mathbb{R}^n,$ then due to additional term $D_{xy}^2$ involved in both ConOpt and ACGD, compared to our method ACOM, there is additional computational cost of the order $O(m^2n^2)$ both for constructing $D_{xy}^{2}$ and for matrix vector operation required (for example, see  steps 2 and 3 of Algorithm 2 in \cite{mescheder2017}).
On the other hand, for ACGD, there are additional cost of order $O(m^3n^3)$ for solving the linear system with matrix of order $mn \times mn$ if direct method is used, and of order $O(mn)$ if iterative methods such as CG as in \cite{schaefer20a} is used. Also, the updates for SGA is more costly due to additional operations (see Algorithm 1 in \cite{balduzzi2018}). For empirically verifying the time complexity, for smaller dataset, in Figure \ref{fig:timechart}, we observe that ACGD is slowest. Remaining methods SGA, OMD, and ConOpt are costlier than our method. Although, ConOpt looks closer, for larger dataset CIFAR10, in Figure \ref{fig:timechart_big}, ConOpt is twice slower than ACOM. The first order methods such as ADAM or SGD are faster per iteration, but qualitatively they never achieve good inception scores. As seen in inception score for CIFAR10, ACGD does not achieve high inception scores early on, hence, it does not have additional advantage compared to ACOM.

\begin{figure}
    \centering
    \begin{subfigure}[t]{0.4\textwidth}
         \includegraphics[width=1.0\linewidth]{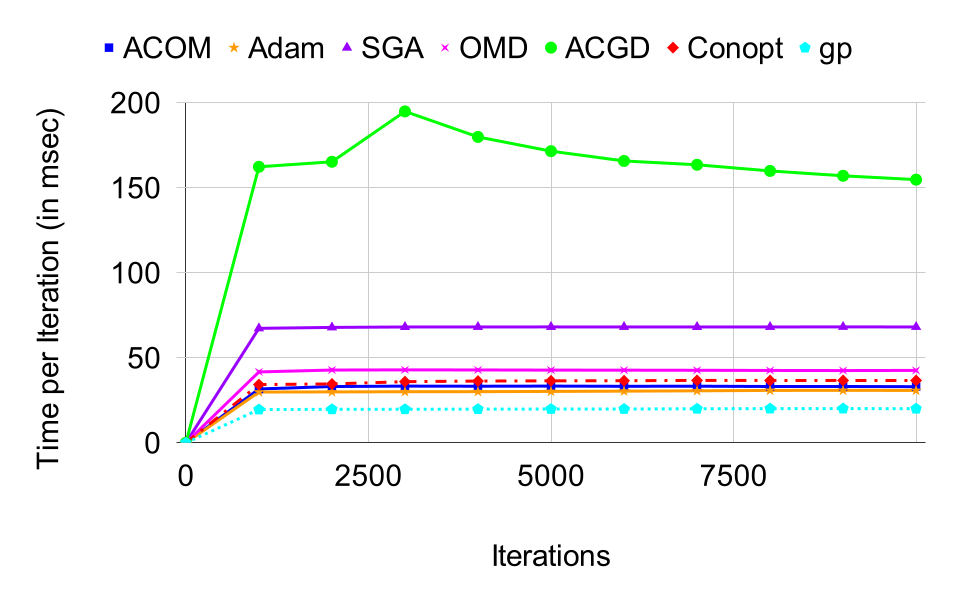}
    \caption{Time Comparision on small dataset.}
    \label{fig:timechart}
    \end{subfigure}
    %
    \begin{subfigure}[t]{0.47\textwidth}
    \begin{center}
        \includegraphics[width=0.8\linewidth]{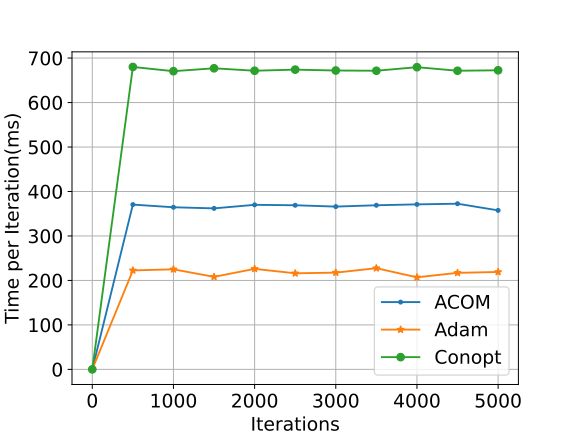}
    \end{center}
    \caption{Time Comparison for large dataset. We don't include ACGD because it is very slow on large datasets. Combining both figures we observe that our method ACOM is fastest second order method.}
    \label{fig:timechart_big}
    \end{subfigure}
\end{figure}

\section{Numerical Experiments.}
\subsection{Experimental setup and machine used.}
Codes: \url{https://github.com/misterpawan/acom}. All experiments were performed on {\tt Intel Xeon E5-2640} v4 processors providing 40 virtual cores, 128GB of RAM and attached to one {\tt NVIDIA GeForce GTX 1080 Ti} GPU providing 14336 CUDA cores and 44 GB of {\tt GDDR5X VRAM}. 
All the codes were written in {\tt python 3.9.1} using {\tt PyTorch (torch-1.7.1)}. For evaluating models, we have used Inception score which is calculated using {\tt inception\_v3} module from {\tt torchvision-0.8.2}. The loss function used is {\tt BCELogitsLoss}. The hyperparameters used for ACOM is mentioned in Algorithm \ref{acom}, for ACGD in Algorithm 1 in \cite{ACGDarXiv}, for ADAM: $\beta_1=0.9, \beta_2=0.999, \alpha=1e^{-3}.$ 
We need not explicitly calculate $D_{xx}^{2}f$ (or $D_{yy}^2 f$) and multiply with $\Delta x$ (or $\Delta y$) in step 4 (or step 11) of Algorithm \ref{acom}, because the whole term $D_{xx}^{2}f\Delta x$ is calculated using PyTorch's {\tt autograd} module.
\tiny 
\begin{algorithm*}[t!]
\caption{\label{acom}ACOM: Adaptive Consensus Optimization Method. Hyperparameters used are $\alpha =2*10^{-4}, \beta_1 = 0.5 , \beta_2=0.99$ and $\epsilon = 10^{-8}$}

\begin{algorithmic}[1]

\REQUIRE $\alpha$: learning rate
\REQUIRE $\beta_1$,$\beta_2$ : Exponential decay rates for the moment estimates
\REQUIRE $f(\theta)$: Stochastic objective function with parameters $\theta$
\REQUIRE $\theta_0$ Initial parameter vector

$m_0 \leftarrow 0$ Initialize $1^{st}$ moment vector

$v_0 \leftarrow 0$ Initialize $2^{nd} $ moment vector

$t \leftarrow 0$ Initialize timestep

\REPEAT

\STATE $t\leftarrow t+1$
\STATE $D_{x} \leftarrow \nabla_x f(x_t,y_t)$ 
\STATE $D_{xx}^{2} \Delta{x} \leftarrow \nabla_x{(D_x)} (x_t - x_{t-1})$ \hfill\COMMENT{Update second order derivative} 
\STATE $D_{t}^x \leftarrow D_{xx}^{2} \Delta{x} + D_{x}$  \hfill\COMMENT{Update first order derivative}
\STATE $m_t^{x} \leftarrow \beta_1 \cdot m_{t-1}^{x} + (1 - \beta_1 ) \cdot D_t^x$ \hfill\COMMENT{Update $x$ momentum term}
\STATE $v_t^{x} \leftarrow \beta_2 \cdot v_{t-1}^{x} + (1 - \beta_2) \cdot ({D_t^x})^2$ \hfill\COMMENT{Update $x$ velocity term} 
\STATE $ \hat{m_t^{x}} \leftarrow m_t^x/(1 - \beta_1^t)$ 
\STATE $ \hat{v_t^{x}} \leftarrow v_t^x/(1 - \beta_2^t)$ 

\STATE $D_{y} \leftarrow \nabla_y f(x_t,y_t)$ 
\STATE $D_{yy}^{2} \Delta{y} \leftarrow \nabla_y{(D_y)} (y_t - y_{t-1})$  \hfill\COMMENT{Update second order derivative}
\STATE $D_{t}^y \leftarrow D_{yy}^{2} \Delta{y} + D_{y}$  \hfill\COMMENT{Update first order derivative}
\STATE $m_t^{y} \leftarrow \beta_1 \cdot m_{t-1}^{y} + (1 - \beta_1 ) \cdot D_t^y$ \hfill\COMMENT{Update $y$ momentum term}
\STATE $v_t^{y} \leftarrow \beta_2 \cdot v_{t-1}^{y} + (1 - \beta_2) \cdot ({D_t^y})^2$ 
\hfill\COMMENT{Update $y$ velocity term}
\STATE $ \widehat{m}_t^{y} \leftarrow m_t^y/(1 - \beta_1^t)$ 
\STATE $ \widehat{v}_t^{y} \leftarrow v_t^y/(1 - \beta_2^t)$ 

\STATE $x_{t+1} \leftarrow x_{t} - \alpha \cdot \widehat{m}_t^x/(\sqrt{\widehat{v}_t^x} + \epsilon)$ \hfill \COMMENT{Update $x$}
\STATE $y_{t+1} \leftarrow y_{t} - \alpha \cdot \widehat{m}_t^y/(\sqrt{\widehat{v}_t^y} + \epsilon)$ \hfill \COMMENT{Update $y$}

\UNTIL{$x_t,y_t$ converged}

\ENSURE $(x,y)$

\end{algorithmic}
\end{algorithm*}
\normalsize

\begin{figure*}[t!]
    \centering
      \begin{subfigure}[t]{0.18\textwidth}
        \centering
        \includegraphics[height=0.9in]{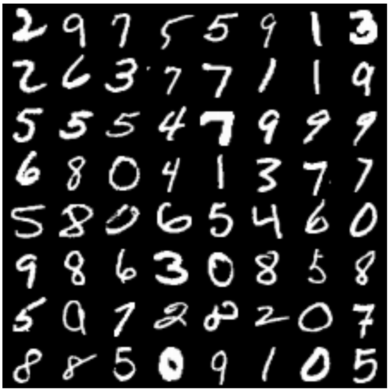}
        \caption{Real}
    \end{subfigure}
    ~
    \begin{subfigure}[t]{0.18\textwidth}
        \centering
        \includegraphics[height=0.9in]{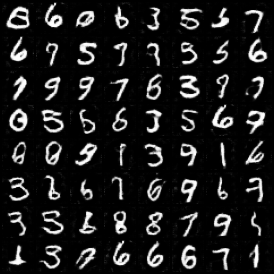}
        \caption{ACOM}
    \end{subfigure}%
    ~ 
    \begin{subfigure}[t]{0.18\textwidth}
        \centering
        \includegraphics[height=0.9in]{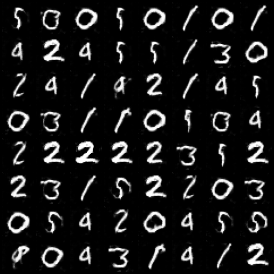}
        \caption{ADAM}
    \end{subfigure}
    ~ 
    \begin{subfigure}[t]{0.18\textwidth}
        \centering
        \includegraphics[height=0.9in]{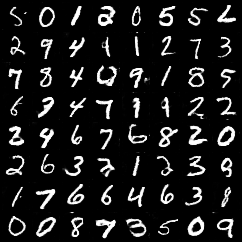}
        \caption{ACGD}
    \end{subfigure}
    ~ 
    \begin{subfigure}[t]{0.18\textwidth}
        \centering
        \includegraphics[height=0.9in]{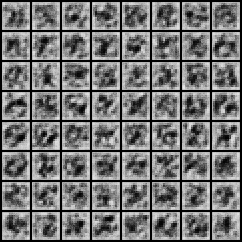}
        \caption{ConOpt}
    \end{subfigure}
    \caption{\label{fig:mnist}Images Generated for MNIST}
\end{figure*}

\begin{figure*}[t!]
    \centering 
    \begin{subfigure}[t]{0.18\textwidth}
        \centering
        \includegraphics[height=0.9in]{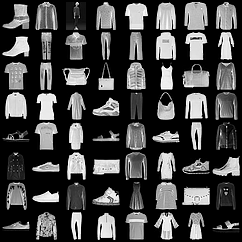}
        \caption{Real}
    \end{subfigure}
    ~
    \begin{subfigure}[t]{0.18\textwidth}
        \centering
        \includegraphics[height=0.9in]{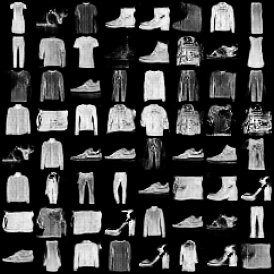}
        \caption{ACOM}
    \end{subfigure}%
    ~ 
    \begin{subfigure}[t]{0.18\textwidth}
        \centering
        \includegraphics[height=0.9in]{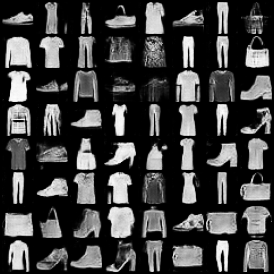}
        \caption{ADAM}
    \end{subfigure}
    ~ 
    \begin{subfigure}[t]{0.18\textwidth}
        \centering
        \includegraphics[height=0.9in]{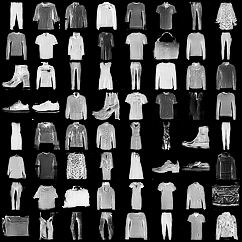}
        \caption{ACGD}
    \end{subfigure}
    ~ 
    \begin{subfigure}[t]{0.18\textwidth}
        \centering
        \includegraphics[height=0.9in]{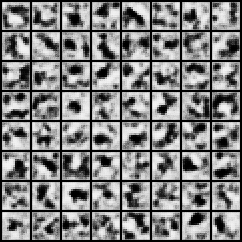}
        \caption{ConOpt}
    \end{subfigure}
    
    \caption{\label{fig:fminst}Generated Images for Fashion MNIST}
\end{figure*}

\begin{figure*}[t!]
    \centering
    \begin{subfigure}[t]{0.18\textwidth}
        \centering
        \includegraphics[height=0.9in]{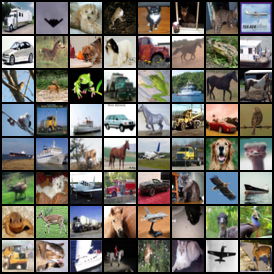}
        \caption{Real}
    \end{subfigure}%
    ~
    \begin{subfigure}[t]{0.18\textwidth}
        \centering
        \includegraphics[height=0.9in]{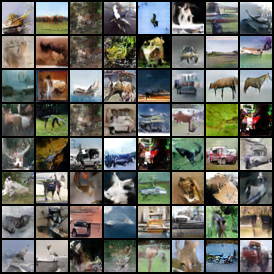}
        \caption{ACOM}
    \end{subfigure}%
    ~ 
    \begin{subfigure}[t]{0.18\textwidth}
        \centering
        \includegraphics[height=0.9in]{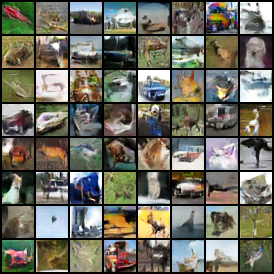}
        \caption{ADAM}
    \end{subfigure}
    ~ 
    \begin{subfigure}[t]{0.18\textwidth}
        \centering
        \includegraphics[height=0.9in]{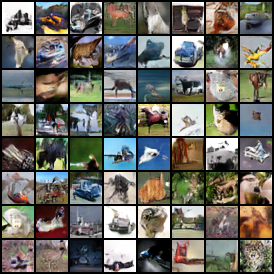}
        \caption{ACGD}
    \end{subfigure}
    ~ 
    \begin{subfigure}[t]{0.18\textwidth}
        \centering
        \includegraphics[height=0.9in]{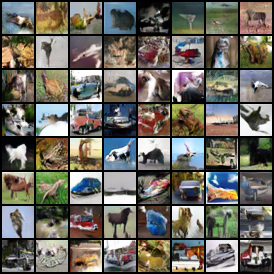}
        \caption{ConOpt}
    \end{subfigure}
    
    \caption{\label{fig:cifar}Generated Images for CIFAR10 \cite{cifar10}.}
    
\end{figure*}

\begin{figure*}[t!]
    \centering
      \begin{subfigure}[t]{0.18\textwidth}
        \centering
        \includegraphics[height=0.9in]{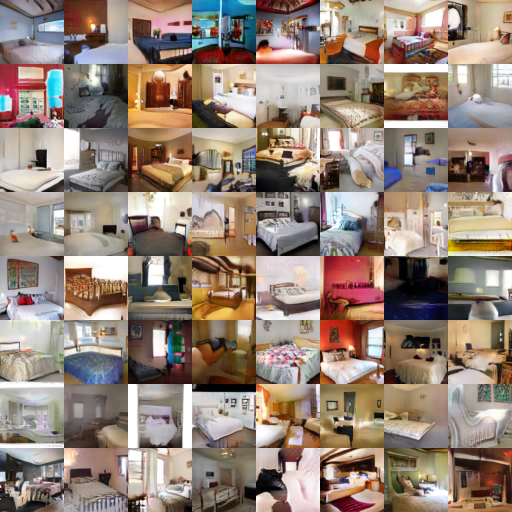}
        \caption{Real}
    \end{subfigure}
    ~
    \begin{subfigure}[t]{0.18\textwidth}
        \centering
        \includegraphics[height=0.9in]{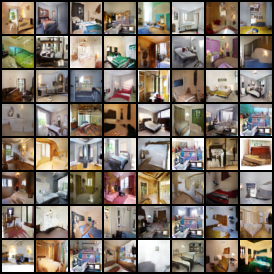}
        \caption{ACOM}
    \end{subfigure}%
    ~ 
    \begin{subfigure}[t]{0.18\textwidth}
        \centering
        \includegraphics[height=0.9in]{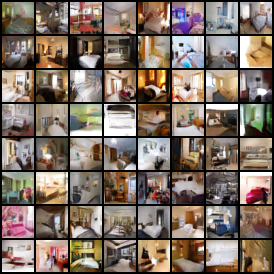}
        \caption{ADAM}
    \end{subfigure}
    ~ 
    \begin{subfigure}[t]{0.18\textwidth}
        \centering
        \includegraphics[height=0.9in]{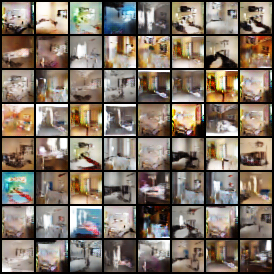}
        \caption{ACGD}
    \end{subfigure}
    ~ 
    \begin{subfigure}[t]{0.18\textwidth}
        \centering
        \includegraphics[height=0.9in]{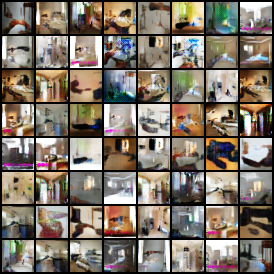}
        \caption{ConOpt}
    \end{subfigure}
    \caption{\label{fig:lsun}Images Generated for LSUN Bedroom \cite{LSUN}.}
\end{figure*}

\begin{figure*}[t!]
    \centering
    \begin{subfigure}[t]{0.33\textwidth}
        \centering
                \includegraphics[height=1.5in]{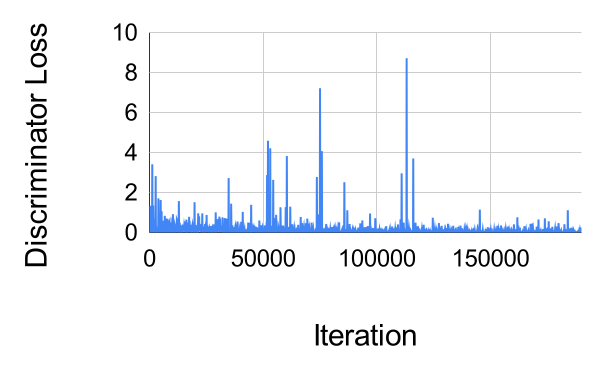}
        \caption{\label{fig:discriminator_loss}Discriminator loss}
    \end{subfigure}%
    ~
    \centering
    \begin{subfigure}[t]{0.33\textwidth}
        \centering
        \includegraphics[height=1.5in]{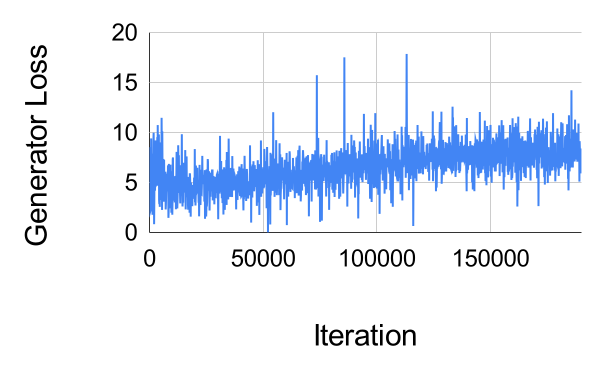}
        \caption{\label{fig:generator_loss}Generator loss}
    \end{subfigure}%
     ~
    \centering
    \begin{subfigure}[t]{0.33\textwidth}
        \centering
        \includegraphics[height=1.5in]{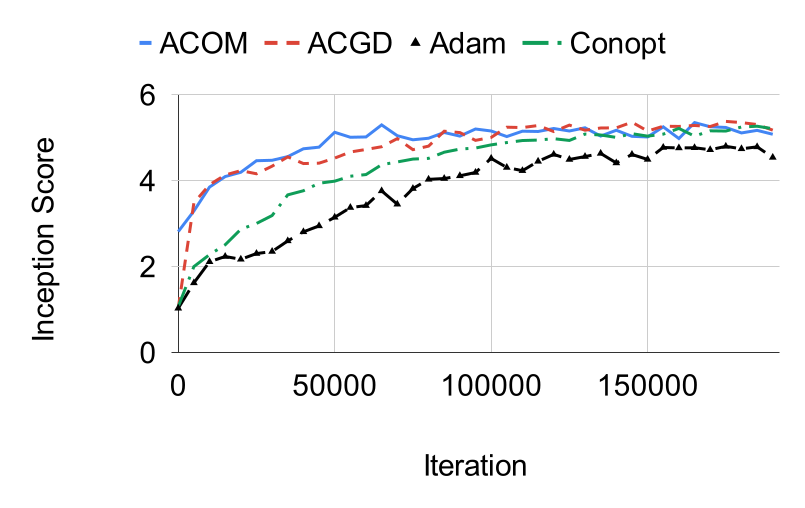}
        \caption{\label{fig:inception}PyTorch Inception Scores on CIFAR10.}
    \end{subfigure}%
    
    \caption{\label{fig:loss} Loss and Inception Scores.}
    
\end{figure*}

\begin{figure*}[t!]
    \centering
      \begin{subfigure}[t]{0.23\textwidth}
        \centering
        \includegraphics[height=0.9in]{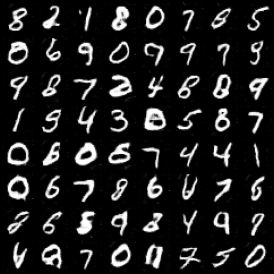}
        \caption{MNIST}
    \end{subfigure}
    ~
    \begin{subfigure}[t]{0.23\textwidth}
        \centering
        \includegraphics[height=0.9in]{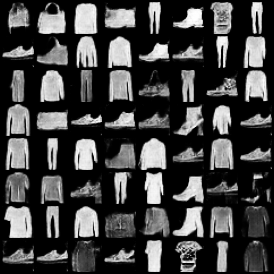}
        \caption{Fashion MNIST}
    \end{subfigure}%
    ~
    \begin{subfigure}[t]{0.23\textwidth}
        \centering
        \includegraphics[height=0.9in]{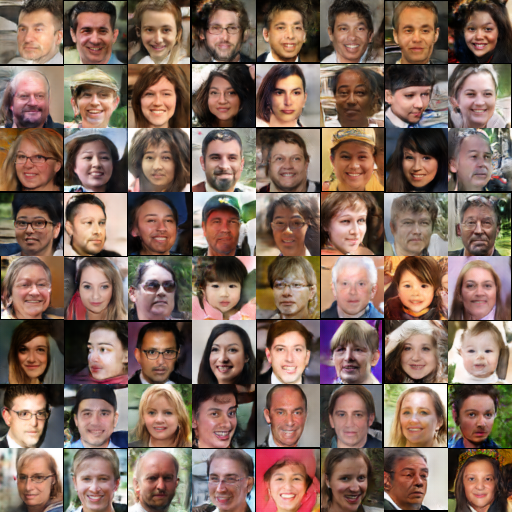}
        \caption{FFHQ on $64 \times 64$ images.}
        \label{fig:ffhq}
    \end{subfigure}
   ~
      \begin{subfigure}[t]{0.23\textwidth}
         \includegraphics[width=1.0\linewidth]{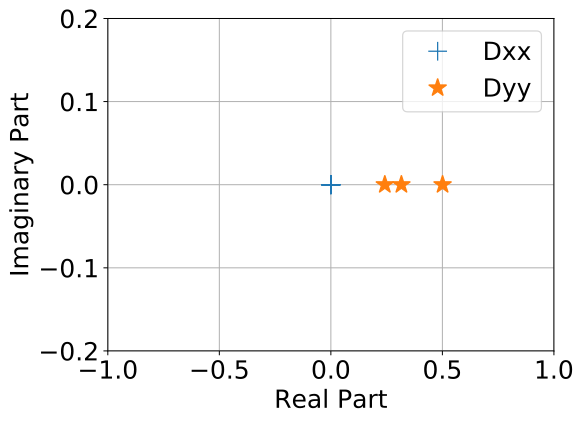}
    \caption{Eigen values of $D_{xx}^{2}f$ and $D_{yy}^{2}f.$}
    \label{fig:eig}
    \end{subfigure}
    \caption{\label{fig:rmsprop} Left three: Generated images of ACOM with RMS prop. Right: Eigenvalue plot.}
\end{figure*}

\subsection{GAN architecture and loss function used.}
We use the {\tt DC-GAN} architecture \cite{radford2016unsupervised} as shown in Table \ref{table:generator} and Table \ref{table:discriminator}. There are various other GAN architectures, and it will be beyond the scope of this work to do extensive comparison with all these. The loss function used is {\tt BCELogitsLoss}. The latent variable $z$ is randomly sampled from standard normal distribution $\mathcal{N}(0,1),$ followed by three convolution layers with {\tt ReLU} activations and block normalizations. 
\begin{table}[h!]
\centering
	\begin{tabular}{c|c|c|c|c}
		\hline
		Module            & Kernel & Stride & Pad & Shape \\
		\hline
		
		Input  & N/A    & N/A    & N/A &  $z \in \mathbb{R}^{100} \sim \mathcal{N}(0, I) $ \\
		Conv, BN, ReLU & $4\times 4$  & 1  & 0    & $100 \rightarrow 1024 $  \\
		Conv, BN, ReLU & $4\times 4$  & 2  & 1    & $1024 \rightarrow 512 $ \\
		Conv, BN, ReLU & $4\times 4$  & 2  & 1    & $512 \rightarrow 256 $  \\
		Conv, Tanh     & $4\times 4$  & 2  & 1    & $256 \rightarrow 3 $  \\
		\hline \\
	\end{tabular}%
	\caption{Generator architecture for CIFAR10 experiments.}
	\label{table:generator}
\end{table}

\begin{table} 
\centering 
	\begin{tabular}{c|c|c|c|c}
		\hline
		Module            & Kernel & Stride & Pad & Shape \\
		\hline
		
		Input  & N/A        & N/A    & N/A &  $z \in \mathbb{R}^{100} \sim \mathcal{N}(0, I) $ \\
		Input  & N/A        & N/A    & N/A &  $x \in \mathbb{R}^{3{\times}32{\times}32}$ \\
		Conv, LeakyReLU     & $4\times 4$  & 2  & 1    & $3 \rightarrow 256 $  \\
		Conv, BN, LeakyReLU & $4\times 4$  & 2  & 1    & $256 \rightarrow 512 $ \\
		Conv, BN, LeakyReLU & $4\times 4$  & 2  & 1    & $512 \rightarrow  1024$  \\
		Conv, Sigmoid       & $4\times 4$  & 1  & 0    & $1024 \rightarrow 1 $  \\
		\hline \\
	\end{tabular}%
	\caption{Discriminator architecture for CIFAR10 experiments}
	\label{table:discriminator}
\end{table}
\normalsize

\subsection{Results for MNIST and Fashion MNIST.}
In Figures \ref{fig:mnist} and \ref{fig:fminst}, we show the generated images for MNIST \cite{lecun-mnisthandwrittendigit-2010} and Fashion MNIST \cite{xiao2017fashionmnist} datasets. We observe that the images generated by ACOM is comparable to real data and those generated by ACGD. We remark here that ConOpt performed poorly for these two dataset, this could be due to batch normalization as mentioned in original paper \cite{mescheder2017}. These datasets were not tested before in ConOpt paper. Similarly, for the other dataset Fashion MNIST, the output generated images from our method is comparable to real sample. For both these datasets, our method was fastest to train. Our proposed method is also effective with RMSprop, in Figure \ref{fig:rmsprop}, we show generated images for ACOM with RMSProp; we observe that the generated images are of similar quality for both ADAM and RMSprop.   

\subsection{Results for CIFAR10, LSUN and FFHQ}

To compare our method on standard CIFAR10 dataset \cite{cifar10}, in subfigures of Figure \ref{fig:cifar}, we compare sample generated datasets from ACOM, ADAM, ACGD, and ConOpt. The images generated by our method ACOM is comparable to real sample. Since for CIFAR10, a standard metric to compare is inception score, in subfigure \ref{fig:inception} in Figure \ref{fig:loss}, we plot the inception scores for these methods, we find that our method achieves high inception score much earlier than both ACGD and ConOpt for the same GAN architecture as mentioned above. For reference, we have also plotted discriminator and generator losses in Figure \ref{fig:discriminator_loss} and \ref{fig:generator_loss} respectively.  However, we must mention here that the first order methods such as ADAM never achieved comparable inception scores as those of second order methods; this observation was also found in previous works such as \cite{schaefer20a,mescheder2017}, where other first order methods were also compared. Hence, existing literature and our comparison suggest that qualitative improvements are seen with second order methods. Also, as shown before, our method is fastest to train and requires less memory among second order methods. In Figure \ref{fig:lsun} and \ref{fig:LSUNbridges3} (more samples only for our method), we compare the generated images for LSUN\footnote{https://www.yf.io/p/lsun} \cite{LSUN} bedroom dataset for $32\times32$ images; we find that the generated images are close to state-of-the-art. Lastly, in Figure \ref{fig:ffhq} and \ref{fig:FFHQ2} (more samples only for our method), we show images generated from FFHQ dataset\footnote{https://github.com/NVlabs/ffhq-dataset}; this dataset offers a significant variety in terms of ethnicity, age, viewpoint, image background, and lighting for face images. To make training feasible on our machines, we down sampled the original images from $128 \times 128$ (thumbnail images) to $64 \times 64.$ We find that the generated images are close to realistic, and we see good diversity in the generated images. Some more generated images from ACOM are shown in Figures \ref{fig:FFHQ2}, \ref{fig:LSUNbridges3} and \ref{fig:LSUNrest3}.

\begin{figure}[t!]
    \centering
        \includegraphics[scale=0.8]{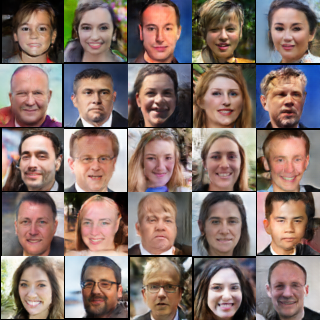}    
    \caption{\label{fig:FFHQ2}Generated Sample for FFHQ Dataset.}
\end{figure}

\begin{figure}[t!]
    \begin{center}
        \includegraphics[scale=0.48]{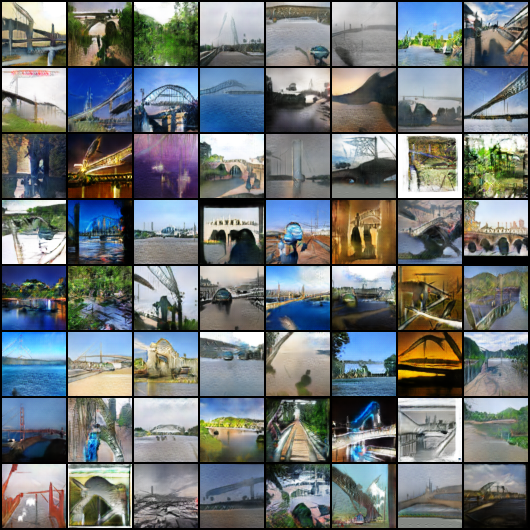}
    \caption{\label{fig:LSUNbridges3}Generated Sample for LSUN Bridges Dataset\cite{LSUN}.}
    \end{center}
\end{figure}

\begin{figure}[h!]
    \centering
        \includegraphics[scale=0.48]{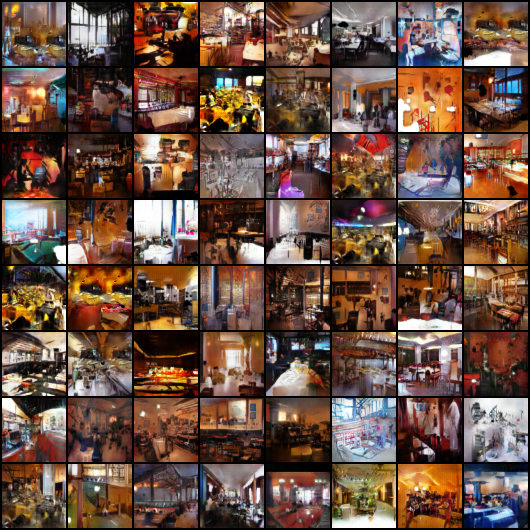}    
    \caption{\label{fig:LSUNrest3}Generated Sample for LSUN Restaurant Dataset\cite{LSUN}.}
\end{figure}

\section{Conclusion.}
We proposed a simple second order update rule, which shows state-of-the-art quality for output images, moreover, it is fastest to train among all recent second order methods. More precisely, in our method, the recent updated gradient is passed on in ADAM or RMSprop method for first and second order momentum calculations. Our method does not involve any mixed derivatives (as in ConOPT) or it does not involve solving costly linear system (as in ACGD) contrary to other recent second order methods. When comparing the well known inception score on standard CIFAR10 dataset, our inception scores is among the best. Our experiments suggest that the mixed derivatives terms in the solver may only be useful for artificial toy example cases, in practice (as well as in theory as proved), for practical datasets (five state-of-the-art datasets) such terms are unnecessary, and using these leads to slow training. We showed a rigorous convergence analysis of the proposed method seen as a fixed point iteration, which to the best of our knowledge is the only complete analysis of the full algorithm (second order update with momentum), which is not done in other existing second other methods. In future, we would like to see how the proposed method behaves for other types of GAN architectures and losses.

\section*{Acknowledgement}
Supported by Qualcomm Faculty Award and MAPG grant.


{\small
\bibliographystyle{ieee_fullname}
\bibliography{mybib}
}

\end{document}